\documentclass[12pt]{l4dc2020} 

\usepackage{todonotes}

\usepackage[ruled]{algorithm2e}

\usepackage{amssymb}
\usepackage{graphicx}
\usepackage{float}
\usepackage{color}
\usepackage{mathrsfs} 
\usepackage{mathtools}
\usepackage{cleveref}

\usepackage{hyperref}

\graphicspath{ {Figures/}}

\definecolor{tb_color}{rgb}{0.1, 0.2, 0.7}

\newcommand{\parDiff}[2]{\frac{\partial #1}{\partial #2}}

\newcommand{\mat}[1]{\begin{matrix}#1\end{matrix}} 
\newcommand{\pmat}[1]{\left(\mat{#1}\right)} 

\title[]{Contracting Implicit Recurrent Neural Networks:\\Stable Models with Improved Trainability}

\usepackage{times}

\author{%
 \Name{Max Revay} \Email{m.revay@acfr.usyd.edu.au}\\
 \Name{Ian Manchester} \Email{ian.manchester@sydney.edu.au }\\
 \addr Australian Center for Field Robotics,\\ Sydney Institute for Robotics and Intelligent Systems, \\ The University of Sydney, \\
 NSW, 2006. Australia
}

\begin{document}

\maketitle

\begin{abstract}
Stability of recurrent models is closely linked with trainability, generalizability and in some applications, safety. Methods that train stable recurrent neural networks, however, do so at a significant cost to expressibility. We propose an implicit model structure that allows for a convex parametrization of stable models using contraction analysis of non-linear systems.
Using these stability conditions we propose a new approach to model initialization and then provide a number of empirical results comparing the performance of our proposed model set to previous stable RNNs and vanilla RNNs. By carefully controlling stability in the model, we observe a significant increase in the speed of training and model performance.

%
\end{abstract}

\begin{keywords}%
	System Identification, Contraction, Stability, Recurrent Neural Network, Vanishing Gradient, Exploding Gradient, Nonlinear Systems, Echo State Network
\end{keywords}

\section*{Notation}Most of our notation is standard. For a matrix $A$, $A\succ0$ or $A\succeq0$ means that $A$ is positive definite or positive semi-definite. Similarly $A\prec0$ or $A\preceq0$ means that $A$ is negative definite or negative semi-definite. We use $\text{vec}(A)$ to refer to the vector obtained by stacking $A$ into a vector. $A \in \mathbb{D}_+$ means that A is a positive definite diagonal matrix. The normal distribution with mean $\mu$ and variance $\sigma^2$ is $\mathcal{N}[\mu, \sigma^2]$ and the uniform distribution between $a$ and $b$ is $\mathcal{U}[a,b]$.

\section{Introduction}
	Recurrent neural networks (RNNs) are a common class of dynamical system used to model sequential data \citep{yi_convergence_2004,mandic_recurrent_2001,graves_supervised_2012}.
	They have been used extensively in areas such as system identification \citep{sjoberg_nonlinear_1995}, learning based control systems \citep{anderson_robust_2007, knight_stable_2011}, natural language processing \citep{zhou_attention-based_2016} and others. 
	Instability of dynamical systems can lead to unpredictable behaviour, and as such, stability should be a consideration when training and deploying RNNs.
	Systems using RNNs have also been proposed in safety critical applications such as autonomous driving \citep{zyner_long_2017}, surgical robotics \citep{mayer_system_2008} or active prosthetics \citep{boudali_predicting_2019}. In such cases, dynamical instability could lead to  injury or even death. As noted by a RAND corporation report on that state of AI: "The current state of AI verification, validation, test, and evaluation (VVT\&E) is nowhere close to ensuring the performance and safety of AI applications, particularly where safety-critical systems are concerned" \citep{tarraf_department_2019}.
	 In practice, there are few approaches to training RNNs with stability guarantees.
	

	The importance of model stability has more subtle implications than just safety; it is also closely related to the difficulty in fitting models.
	Training recurrent models using gradient descent is complicated by the exploding and vanishing gradients problem \citep{pascanu_difficulty_2013}.
	If a model is too stable the gradients will vanish, and if it is unstable they will explode.
	A common approach to this problem, used for instance in the Long Short Term Memory (LSTM) \citep{hochreiter_long_1997} or Gated Recurrent Unit (GRU) \citep{cho_learning_2014}, is to alter the model structure to one less susceptible to these issues. 
	Alternatively, one can initialize or constrain the weights to be isometries that do not change the magnitude of the gradients. 
	For instance, it has been reported that models such as the iRNN \citep{le_simple_2015}, orthogonal RNN \citep{mhammedi_efficient_2017} and spectral RNN \citep{zhang_stabilizing_2018} can offer similar performance and trainability to the LSTM with fewer parameters.
	An alternative approach  avoids this issue by initializing a sufficiently rich bank of dynamics so that only the input and output mappings must be learned. This is the basis for the Laguerre filter \citep{wahlberg_system_1991} and the echo state network \citep{jaeger_adaptive_2003}.
	These two approaches suggest that a more effective initialization scheme may be to initialize with a rich set of dynamics, of which some modes have long memory.
	
	Another area closely linked to stability is model generalization.
	Empirically, it has been observed that stability is an effective regularizer in system identification \citep{umenberger_convex_2019,umenberger_specialized_2019}. 
	There are also theoretical results that relate generalization to a form of stability \citep{zhang_stabilizing_2018}.	

\subsection{Contraction Analysis}

Stability can be defined in many ways. For input-output systems, \cite{zames_input-output_1966} argued for two properties: Firstly, bounded inputs should produce bounded outputs. Secondly, outputs should not be critically sensitive to small changes in inputs.

Many approaches to RNN stability analysis focus on global stability of a particular equilibrium. This however, guarantees neither of these properties \citep{sontag_input_2008}. An additional complication is that the analysis is centred on a particular trajectory and in practice, we do not know the trajectories for unknown inputs.
For instance, \cite{kaszkurewicz_robust_1993,kaszkurewicz_class_1994, kaszkurewicz_matrix_2000} propose diagonal stability. For a certain class of non-linearity, less conservative stability guarantees can be found using diagonally dominant Lyapunov functions \citep{chu_bounds_1999}. Absolute stability theory \citep{barabanov_stability_2002} has also been used to reduce reduce conservatism. 


Contraction \citep{lohmiller_contraction_1998} and incremental stability \citep{angeli_lyapunov_2002}, on the other hand, both guarantee stability in the sense of \cite{zames_input-output_1966} and are independent of the input or equilibrium.
We provide a brief introduction to contraction analysis. Additional details can be found in \cite{lohmiller_contraction_1998}.
Suppose that we have a non-linear dynamical system with  dynamics:
\begin{equation} \label{eq:contraction_dynamics}
x_{k+1} = f(x_k, k),
\end{equation}
where $x_k$ is the state of the system at time $k$. If $f$ is piecewise differentiable, then we can study the differential dynamics given by:
\begin{equation} \label{eq:differential_dynamics}
\delta_{x_{k+1}} = F(x_k, k, \delta_{x_k})
\end{equation}
where $F(x, k, \delta_x)$ is the directional derivative of $f(\cdot)$ at $x_k$ in the direction $\delta_k$. If $f(\cdot)$ is differentiable then this can be written as $F(x, k, \delta_x) = \parDiff{f(x, k)}{x} \delta_x$. The vector $\delta_x$ can be interpreted as a infinitesimal displacement between two neighbouring trajectories. 
Stability of the differential dynamics \eqref{eq:differential_dynamics} imposes a strong type of stability on the dynamical system \eqref{eq:contraction_dynamics}, whereby, all trajectories of the original system \eqref{eq:contraction_dynamics} converge exponentially to a single trajectory. This is done by searching for a contraction metric or differential Lyapunov function \citep{forni_differential_2014} $V(x_k, \delta_{x_k}) >0, ~\forall ~\delta_x \neq 0$ such that $V(x_{k+1}, \delta_{x_{k+1}}) \leq \lambda V(x_k, \delta_{x_{k}})$ for $0 < \lambda < 1$. In this work we also allow for the case where $\lambda=1$ to allow for non-expansive systems.

In principle there are many metric structures that can be used. A common approach however, parametrizes a quadratic form 
$V(\delta_{x_k}) = \delta_k^\top M \delta_k$,  where $M\succ 0$ is a positive definite matrix. In this case a sufficient condition for contraction when $f(x,k)$ is differentiable is:
\begin{gather} \label{eq:contraction_LMI}
\parDiff{f}{x}^\top M \parDiff{f}{x} - \lambda M \preceq 0, 
\end{gather}
and $M$ is called a contraction metric. Methods that analyse stability by bounding the maximum singular value (e.g \cite{miller_stable_2018,zhang_stabilizing_2018}) can be seen as a special case of \eqref{eq:contraction_LMI}, where $M = I$ and $\lambda = 1$.
The use of a parametrized metric provides considerable flexibility to the model set. This can be seen in the following example:
\begin{example} \label{ex: simple rnn}
	Consider the simple $1$ layer RNN:
	\begin{gather}
	h_1^{k+1} = \sigma(A h_2^{k}), \quad  A = \pmat{0.8 & 1\\ 0 & 0.8}
	\end{gather}
	Where $\sigma(z) = \max(0, z)$ is a ReLU non-linearity with Lipschitz constant $L_\sigma = 1$. The matrix $A$ has a maximum singular value of $1.44$ so it does not satisfy the condition used by \cite{miller_stable_2018}. On the other hand, using condition \eqref{eq:contraction_LMI}, we can construct a contraction metric $V = \delta_{h_k}^\top P\delta_{h_k}$ with $P = \text{diag}(1, 10)$ in which the system is contracting.
\end{example} 

\subsection{Convex Parametrizations} 
Contraction analysis is a powerful tool for studying dynamical systems. Synthesis (e.g. control design or system identification) with contraction constraints is complicated by  non-convexity in the model parameters and contraction metric. 
To be precise, \eqref{eq:contraction_LMI} is convex in $M$ or $f$, but not $M$ and $f$. 
Even when minimizing non-convex loss function, there is benefit to convex stability constraints
If the constraints are convex (even if the objective is not), then the situation is greatly improved. If using a projected gradient method, convexity allows for easy projection onto the feasible set at each step. Alternatively, penalty or barrier functions for convex sets can be added without making the resulting problem much harder than the unconstrained problem.

Methods such as \cite{miller_stable_2018} essentially avoid this problem by fixing the metric at the cost of model expressibility.
It has been found, however, that using an implicit model structure allows parametrizations jointly convex in the model parameters and contraction metric \citep{tobenkin_convex_2017}. Our work extends this approach to the model class of RNNs.

\subsection{Contributions}
 We propose a class of contracting implicit recurrent neural network that is jointly convex in the model parameters and stability certificate. 
The proposed set has less conservative stability conditions which leads to greater expressibility when compared to previous stable RNNs - particularly in the multilayer case. Additionally, we propose an initialization procedure that ensures both a rich set of dynamics and that the model is not too stable to train which improves the model trainability. We then provide empirical results on a simulated model and a gait prediction task highlighting the benefits of our approach.

\section{Model Set}

	We are interested in fitting state space models parametrized by $\theta\in\Theta \subseteq \mathbb{R}^p$, of the following form:
	\begin{gather}\label{eq:ss model}
		x_{k+1} = f_\theta(x_k, u_k) \\
		y_k = g_\theta(x_k, u_k) \label{eq:ss model out}
	\end{gather}
	where the function $f_\theta(x,u)$ can be represented by an $L$-layer Neural Network with skip connections, and $\Theta$ refers to a convex set of parameters to be defined later. In this case, we write the dynamics in \eqref{eq:ss model} as follows:
		\begin{align}
	&z^0 = x, & z^{\ell+1} = \phi(A_\ell z^\ell + B_\ell u + b_\ell)~ \quad\text{for} ~ \ell = 0,..., L-1, & \quad f_\theta(x, u) = z^{L}.& 
	\label{eq:explicit dynamics}
	\end{align}
	Here, the superscript to refers to the layer in the network, $z^\ell$ is the output of the $\ell$'th hidden layer of the network and are not necessarily of the same size. The weight matrices and bias for the $\ell$'th layer are denoted $A_\ell$, $B_\ell$ and $b_\ell$ respectively. 
	
	We define the set of admissible activations $\phi(\cdot)$ to be the set of piecewise differentiable, scalar, non-linearities with slope restricted to $[-\gamma, \gamma]$. For simplicity we will assume $\gamma=1$, however this can be relaxed. This includes any collection of standard activation functions.


	 Skip connections from the input to the $\ell$'th hidden are contained within $B_\ell$ and skip connections between hidden layers are included by replacing part of the activation with a linear activation. 
	 
	 The contraction properties are independent of the output mapping so that $g_\theta(x,u)$ can be any function. In all examples we will take the output to be linear in the input and final hidden layer so that $g_\theta(x,u)= C x + D u$.

	 \subsection{Implicit RNNs}
	 We will refer to the dynamics in \eqref{eq:explicit dynamics} as the explicit model. We can also  parametrize the same set of models using the following implicit, redundant parametrization:
	 	 \begin{align}\label{eq:RNN}
	 &E_0 h^0 =  x,	 &E_{\ell+1} h^{\ell+1} = \phi(W_\ell h^{\ell}+ B_\ell u + b_\ell)~ \text{for}~\ell = 0,...,L-1,  &\quad f_\theta(x, u) = E_L h^{L},&
	 \end{align}
	 where $W_\ell$ and $E_\ell$ are learnable weight matrices and $E_\ell$ are invertible. 	
	 Note that the implicit and explicit models are input/output equivalent under the coordinate transformation $z_{\ell} = E_{\ell}h_\ell$ and $A_\ell = W_\ell E_{\ell-1}^{-1}$. 

	
	We can treat multi-layer networks as a time-varying, periodic, non-linear system  by dividing up each $k$ step into $L$ sub-steps so that 
	 \begin{equation}
	 h_k^{\ell+1} = f^\ell(h_k^\ell, u_k), \ \ \ell = 0,...,L-1
	 \end{equation}
	with $f^\ell$ defined in \eqref{eq:RNN} and $h^0_{k+1} = h^{L}_k$.
	The associated differential dynamics of the network are given by 
$E_{\ell+1} \delta_{{k}}^{\ell+1} =  \Lambda\left(h_k, W_\ell \delta_{k}^\ell\right) , \ \ell = 0,...,L-1, $
where $\Lambda(h_k, \delta_k)$ is the directional derivative of $\phi$ at $h_k$ in the direction $\delta$ and $\delta_{k}^\ell$ is a differential in the at time $k$ and layer $\ell$.




\subsection{Contraction Implicit RNNs}
We now define the set of contracting implicit RNNs (ci-RNNs): A ci-RNN is an implicit RNN defined as \eqref{eq:ss model} with $f_\theta(x,u)$ defined in \eqref{eq:RNN} with an additional contraction constraint. We propose to use the following constraints to ensure model stability:
\begin{gather} \label{eq:RNN_single_layer_LMI}
\begin{pmatrix}
E_\ell + E_\ell^\top -  P_\ell &  W_\ell^\top \\
W_\ell &  P_{\ell+1}
\end{pmatrix} \succeq 0, \quad \ell = 0,...,L-1
\end{gather}
with $P_0 = \lambda P_L$. The set of ci-RNNs, denoted $\Theta_{ci}$ is defined as:
$$ \label{eq:ciRNNs}
\Theta_{ci}:= \left\{\theta~:~\exists P_0,...,P_L \in \mathbb{D}_+ ~ \text{s.t.} ~ P_0 = \lambda P_L,~ E + E^\top \succ 0, ~ \eqref{eq:RNN_single_layer_LMI}\right\}
$$
Note that $\Theta_{ci}$ is convex as it is the intersection a number of semi-definite cones and a linear equality constraint, and for all $\theta\in \Theta_{ci}$, there exists a corresponding explicit RNN \eqref{eq:explicit dynamics}. Fixing $E_\ell=I$ and $P_\ell=I$ recovers the model set used by \cite{miller_stable_2018}.

\begin{theorem} \label{thm:RNN_single_layer_storage_function}
	Suppose that $\theta \in \Theta_{ci}$, then the model \eqref{eq:ss model}, \eqref{eq:RNN} is contracting with rate $\lambda$ in the metric $V=\delta_x^\top E_0^\top P_{0}^{-1} E_0 \delta_x$.
\end{theorem}
\begin{proof} We would like to show that the condition \eqref{eq:RNN_single_layer_LMI} implies the existence of a contraction metric $V_k$ for the system \eqref{eq:ss model}, \eqref{eq:RNN}, for which $V_{k+1}\leq \lambda V_k$.
	Via Schur complement \eqref{eq:RNN_single_layer_LMI} is equivalent to:
	$$
	E_\ell + E_\ell^\top -  P_\ell  - 
	W_\ell^\top  P_{\ell+1}^{-1} W_\ell \succeq 0.
	$$
	For all admissible activation functions and diagonal $P \succ 0$, we have $\delta^\top P^{-1} \delta \geq \Lambda(h, \delta)^\top P^{-1} \Lambda(h, \delta) $. Left and right multiplying by $\delta_h$ gives 
	$$
 \delta_{h_{\ell+1}}E_{\ell+1}^\top P_{\ell+1}^{-1} E_{\ell+1} \delta_{h_{\ell+1}} - \delta_h'(E + E' - P)\delta_h \leq 0.
	$$
	Introducing the storage function $V_k^\ell(\delta_{h_k}^\ell) = {\delta_{h_{k}}^\ell}^\top E_{\ell}^\top P_{\ell}^{-1} E_{\ell} \delta^\ell_{h_{k}}$ and using the bound $E^\top P^{-1} E \succ E + E^\top - P$, we can see that $		V_k^{\ell+1} - V_k^\ell < 0$.
	Summing this from $\ell = 0,....,L-1$ gives $V_k^L - V_k^0 \leq 0$. Due to the periodicity, we have $V_{k+1}^0 =  \lambda V_k^L$, so $V_{k+1}^0 \leq \lambda V_k^0$, and the system is contracting in the metric $V_k^0$.
\end{proof}

\section{Method} \label{sec:Method}
We demonstrate the use of the proposed model set in a system identification context. 
In particular, we are interested in finding functions $f_\theta$ and $g_\theta$ that minimize the simulation error:
\begin{equation} \label{eq:simulation error}
\min_{\theta\in\Theta, h_0} J_{sim} = \sum_{k=0}^{T} |y_k - \tilde{y}_k|^2 \quad 
s.t. \quad 	 h_{k+1} = f_\theta(h_k, \tilde{u}_k), \quad y_k = g_\theta(h_k, \tilde{u}_k)
\end{equation}
where $\Theta$ is the domain of the parameters and $(\tilde{u}_k, \tilde{y}_k)$ are the measured inputs and outputs to system we would like to identify. We will compare the proposed ci-RNN with two others. The first is a regular RNN defined by the equations \eqref{eq:explicit dynamics} and the second is the stable RNN (s-RNN) defined by the explicit dynamics \eqref{eq:explicit dynamics} with $A_\ell$ having spectral norm less than 1. We enforce this using the following LMIs:
\begin{equation} \label{eq:sn_stability}
		\pmat{I & A_\ell^\top \\ A_\ell & I } \succeq 0, \quad \ell = 0,...,L-1.
	\end{equation}
In the one layer case, this is the same model set used in \cite{miller_stable_2018}.
%

\subsection{Model Initialization} \label{sec:Initialization}
We propose to initialize the models in a two step procedure. Firstly, we sample weights for the explicit model \eqref{eq:explicit dynamics} as follows
\begin{equation} \label{eq:weight sampling}
A^{ij} \sim \mathcal{N}\left[0, \frac{\alpha^2}{n}\right],
\end{equation}
where $\alpha$ is a hyper parameter that relates to how close to instability we expect our model to operate and $n$ is the width of the weight matrix.  According to the random matrix circular law \citep{tao_random_2008}, we expect the eigenvalues to be approximately distributed over a circle of radius $\alpha$. The intuition is to try and generate a rich set of dynamics so that we only need to learn the input and output mappings of the dynamical system.
Depending on $\alpha$, we may find that this method of sampling generates a number of unstable models that complicate training. We project onto the set of contracting implicit models by solving the following convex optimization problem:
\begin{align}
& \underset{\theta \in \Theta_{ci}}{\min} \quad \sum_{\ell=0}^{L-1}|A_\ell E_\ell - W_\ell|^2_F
\end{align}
We solve this optimization problem using the cvxpy toolbox \citep{diamond_cvxpy:_2016}.
For the s-RNN model set, we project onto the set of stable models by clipping the singular values as in \cite{miller_stable_2018}. We leave the regular RNNs as they are.

\subsection{Training Procedure}

Fitting the models ci-RNN or s-RNN require a number of LMIs to be satisfied. We do this using the Burer-Montero method that has been shown to be both empirically \citep{burer_nonlinear_2003} and theoretically \citep{boumal_non-convex_2016} effective for a wide range of problems. This involves replacing a semi-definite constraint $M\succeq 0 $ with a series of equality constraints:
$ M = \mathcal{L} \mathcal{L}^\top $
where $\mathcal{L}$ is an auxiliary matrix variable. 

We then use ADAM optimizer to minimize the following objective 
$
J = MSE(y_{0:T}, \tilde{y}_{0:T}) + \mu c^\top c
$ where MSE is the mean square error, $y_{0:T}$ are the outputs from simulating the model using \eqref{eq:ss model} and \eqref{eq:ss model out} and $c = \text{vec}(M - \mathcal{L}\mathcal{L}^\top)$ are the equality constraints. We use an initial learning rate of $0.5\times 10^{-3}$ which decays by a factor of $0.96$ at each epoch and an initial penalty parameter of $500$. If the equality constraints are violated by more than $1\times10^{-3}$, we increase the penalty parameter by a factor of $10$.
The models are trained until more than $20$ epochs have passed without seeing a model better than the best seen so for (on validation).

\section{Results}
We test the proposed approach on two systems. Firstly, we will look at a simple simulated system and explore the effects of implicit parametrizations, stability constraints and model initialization.
Then, we will compare the ci-RNN to both the RNN and s-RNN using a human gait prediction task based on data gathered from Motion Capture (MOCAP) experiments. Code to reproduce all examples can be found at \url{https://github.com/imanchester/ci-rnn}.

\subsection{Simulated System} \label{sec:chen's}
We generate data from a slight modification of the system used in \cite{chen_non-linear_1990}. The system has been modified so that it operates closer to the edge of stability, as follows:
\begin{multline} \label{eq:chen's_system}
x_k = 1.4 \bigg[ \left(0.8 - 0.5e^{-x_{k-1}^2}\right) x_{k-1} - \left(0.3 + 0.9e^{-x_{k-1}^2}\right)x_{k-2} + \\ u_{k-1} + 0.2 u_{k-2} + 0.1 u_{k-1} u_{k-2} + w_k \bigg]
\end{multline}
with process noise $w_k \sim \mathcal{N}(0, 0.5)$ and inputs $u_k \sim \mathcal{N}(0, 1)$. For each model realization, we generate a data set consisting of 20 batches of 500 measurements and train models with 2 layers of 60 hidden units per layer  and ReLU activations. We train 5 different types model denoted A-E. The models `A' are ci-RNNs with the initialization scheme in Section \ref{sec:Initialization}, $\alpha=1.2$. `B' are implicit models with the same initialization but without the stability constraint \eqref{eq:RNN_single_layer_LMI}. `C' are implicit models initialized so that $E=I$ and $W_{ij} \sim \mathcal{U}[-\frac{1}{\sqrt{60}}, \frac{1}{\sqrt{60}}]$. `D' are explicit models with no stability constraints initialized by sampling $A_{ij}\sim \mathcal{N}[0, \frac{1}{{60}}]$ and finally, `E' are explicit models initialized by sampling $A_{ij}\sim \mathcal{N}[0, \frac{1}{{60}}]$ and projecting onto the unit spectral norm ball.
\begin{figure}
	\hfil
	\begin{minipage}[]{0.45\linewidth}
	\centering
	\includegraphics[trim = {0.5cm 7cm 1.5cm 7.5cm}, clip,width=
	0.8\linewidth]{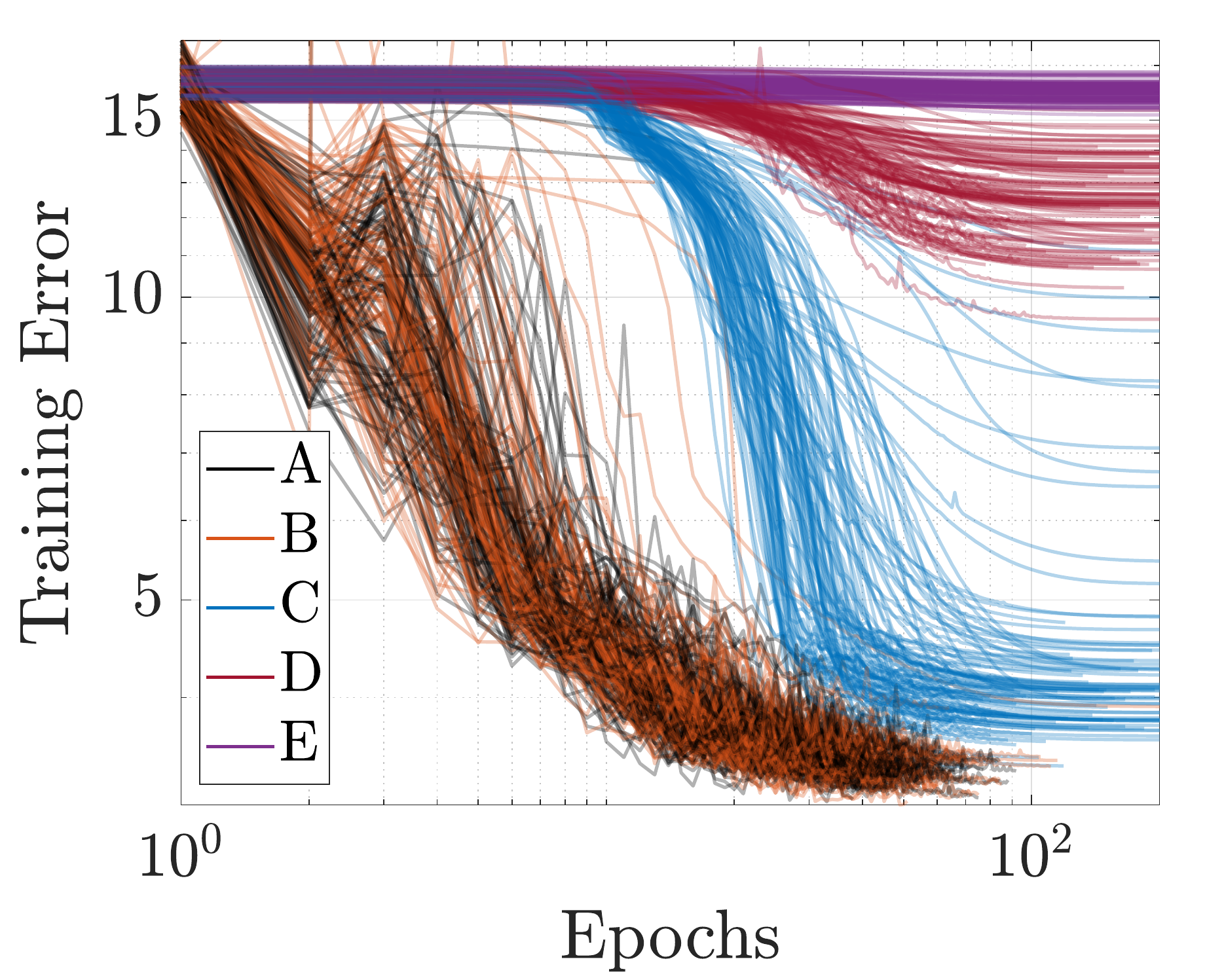}
	\caption{\label{fig:chen train}Best viewed in colour. Figure explained in text (Section \ref{sec:chen's}).}	
\end{minipage}
\hfill
\begin{minipage}[]{0.45\linewidth}
	\centering
	\includegraphics[trim = {0.5cm 7cm 1.5cm 7.5cm}, clip,width= 0.8\linewidth]{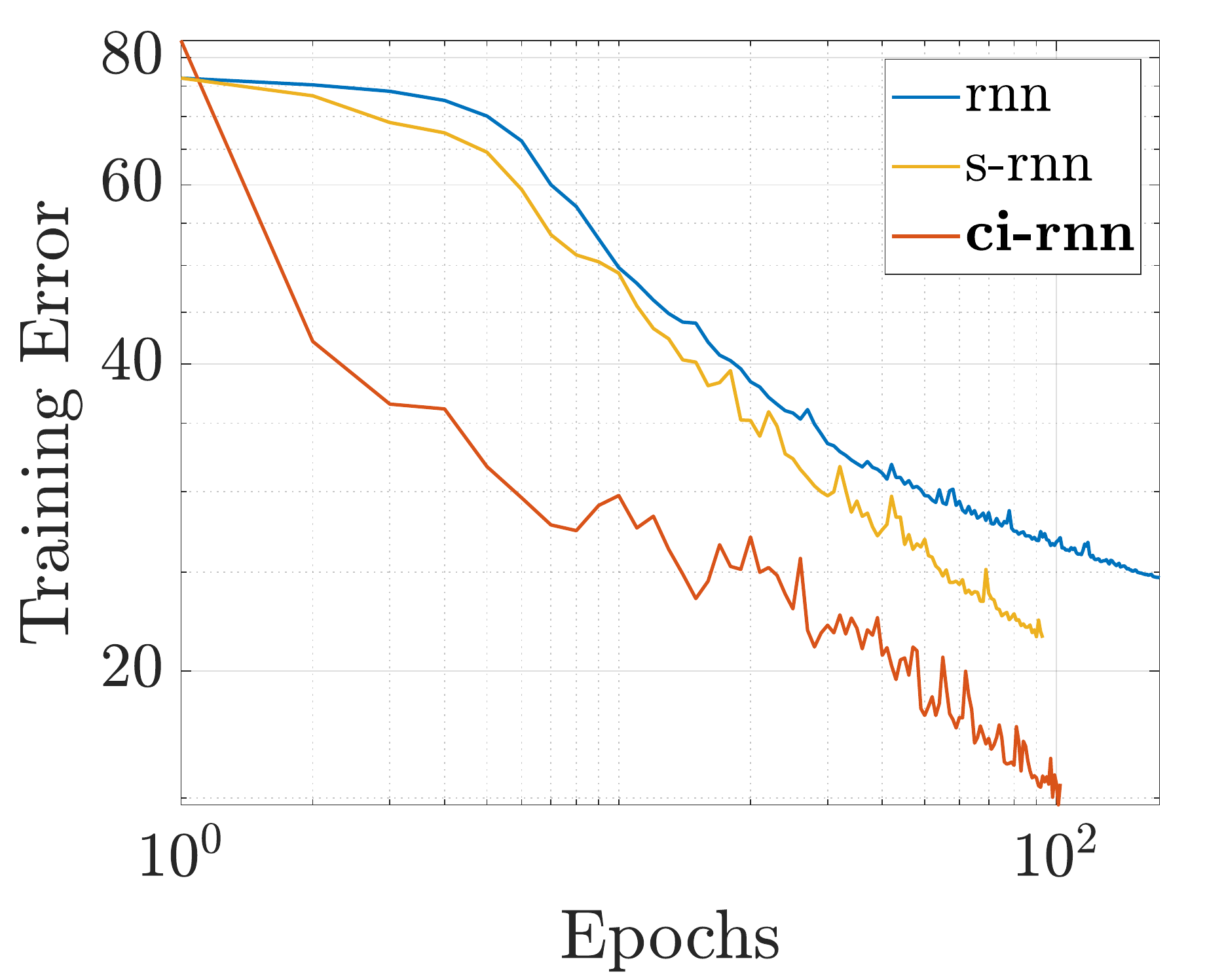}
	\caption{\label{fig:training speed}  Training error versus epochs for gait prediction task for subject 1.}
\end{minipage}
\hfil
\end{figure}

Comparing the models `C' with `D' and `E', we can see that the implicit model structure appears to make training much easier. Comparing the models `A' and `B' with `C', we also see that the initialization procedure appears to significantly speed up training. 
Finally, comparing A with B we see that the proposed contraction constraint did not hinder training compared to unconstrained models of the same structure. This is in contrast to the spectral norm constraint: comparing D and E we can see that the constraint dramatically hinders training.
\subsection{Gait Prediction}
 The problem is to determine a mapping to the trajectory of the left leg joint angles from the trajectories of the remaining limbs. Such a model can be used, for example, to generate trajectories for an actuated prosthetic limb or exoskeleton. As noted by the authors, this is a system where stability is an important concern as unstable models can lead to unpredictable or dangerous behaviour.

The problem data consists of measurements of joint angles from $9$ participants who instructed to walk across flat ground, up a flight of stairs and then stop at the top. Data was gathered using a MOCAP system.
Additional details on the data collection can be found in \cite{boudali_prediction_2019,boudali_predicting_2019}.  The exercise was repeated $11$ times for each participant and the final two trials were withheld as a test set. The remaining $9$ datasets were used to train $9$ models using $9$-fold cross validation. All models have $64$ hidden units per layer, ReLU activations and layers varying from 1 to 3.

\begin{figure}[]
	\centering
	\subfigure[Subject 1 Training NSE\label{fig:train boxplot}]{\includegraphics[trim = {4.1cm 1.5cm 6cm 0.5cm}, clip,width=0.32 \linewidth]{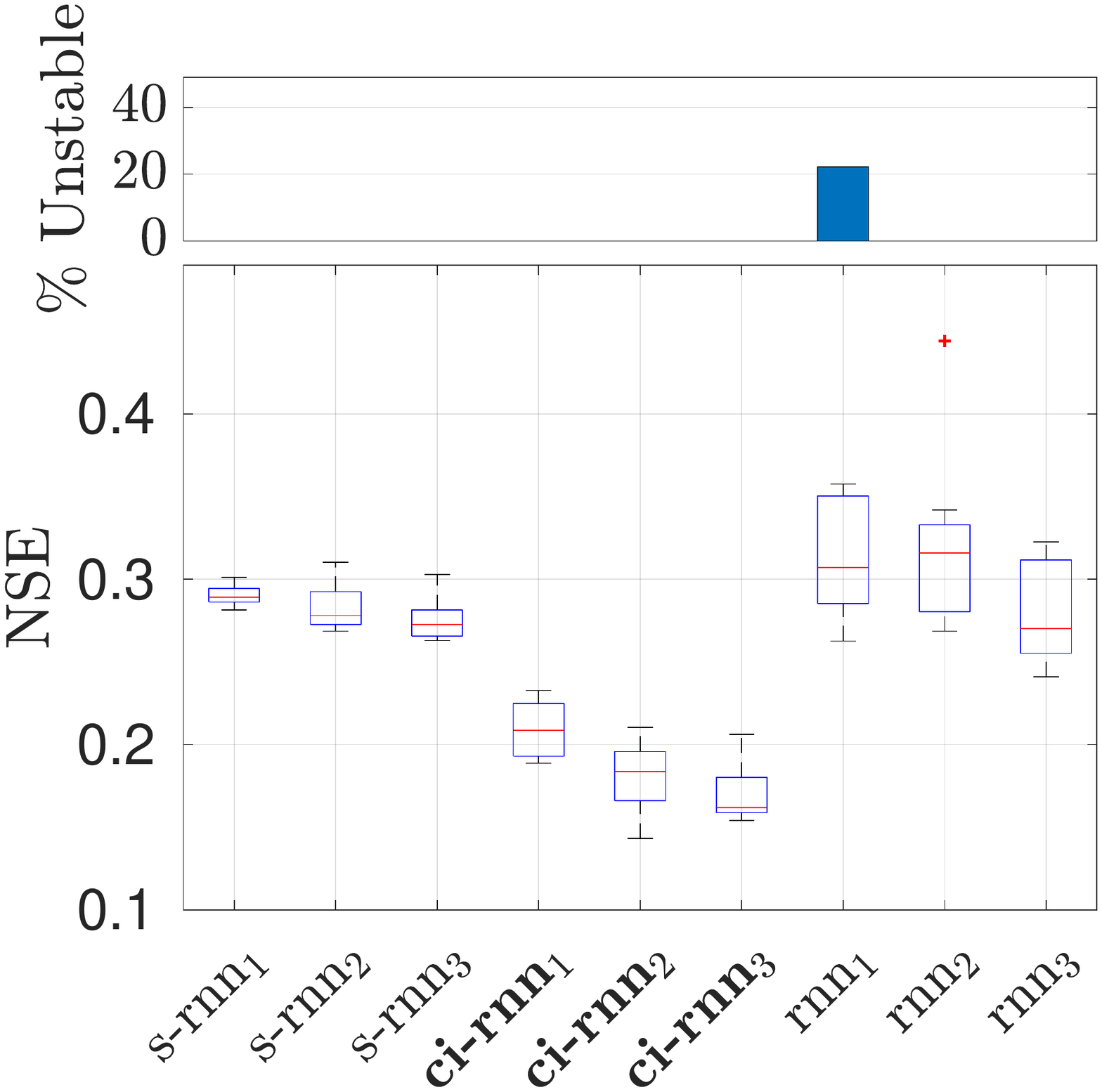}}
	\hfill
	\subfigure[Subject 1 Test NSE\label{fig:test boxplot}]{\includegraphics[trim = {4.1cm 1.5cm 6cm 0.5cm}, clip,width=0.32 \linewidth]{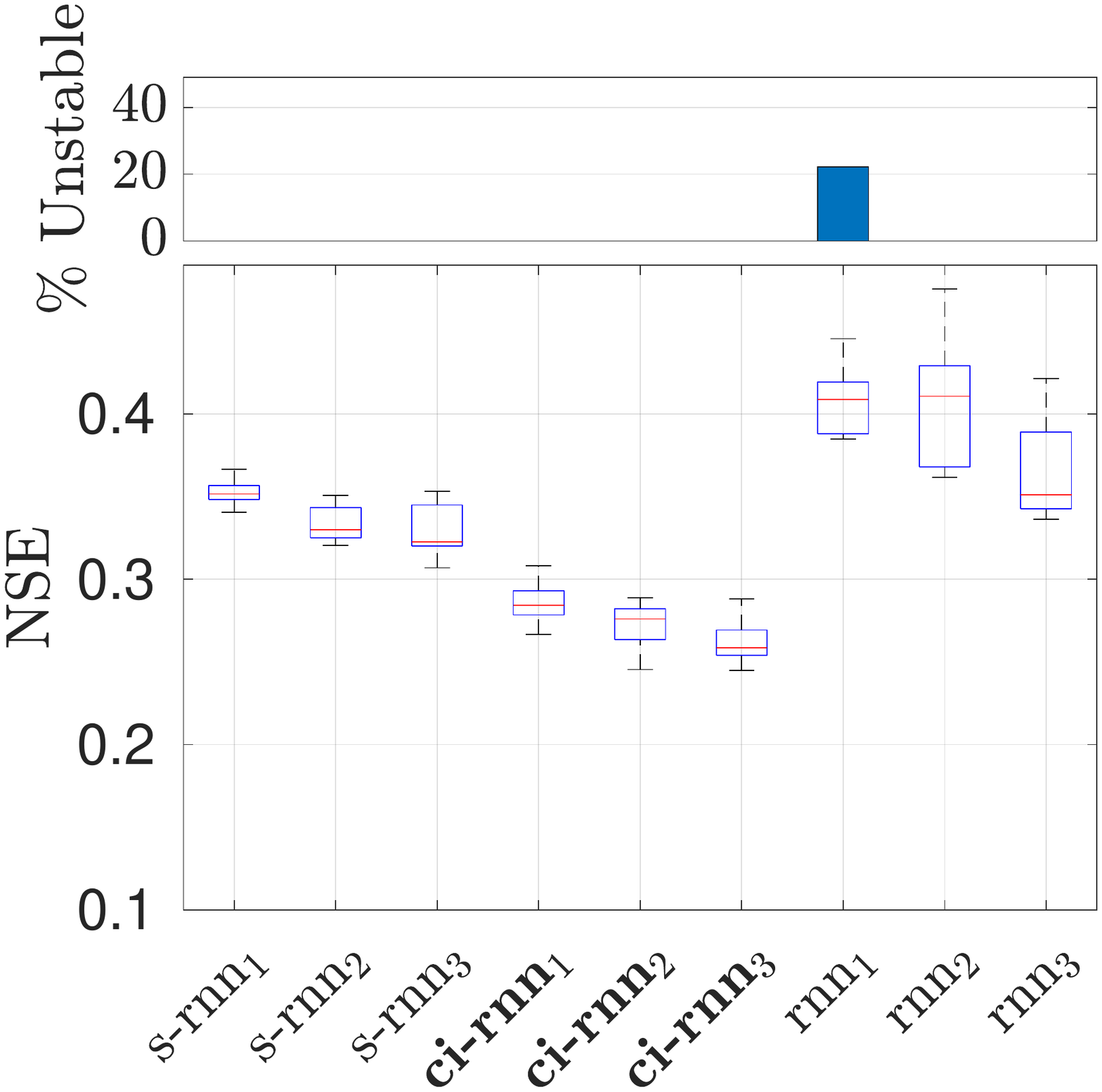}}
	\hfill
	\subfigure[All subjects Test NSE \label{fig:test scatter}]{\includegraphics[trim = {0.7cm 6cm 2.0cm 7cm}, clip,width=0.32\linewidth]{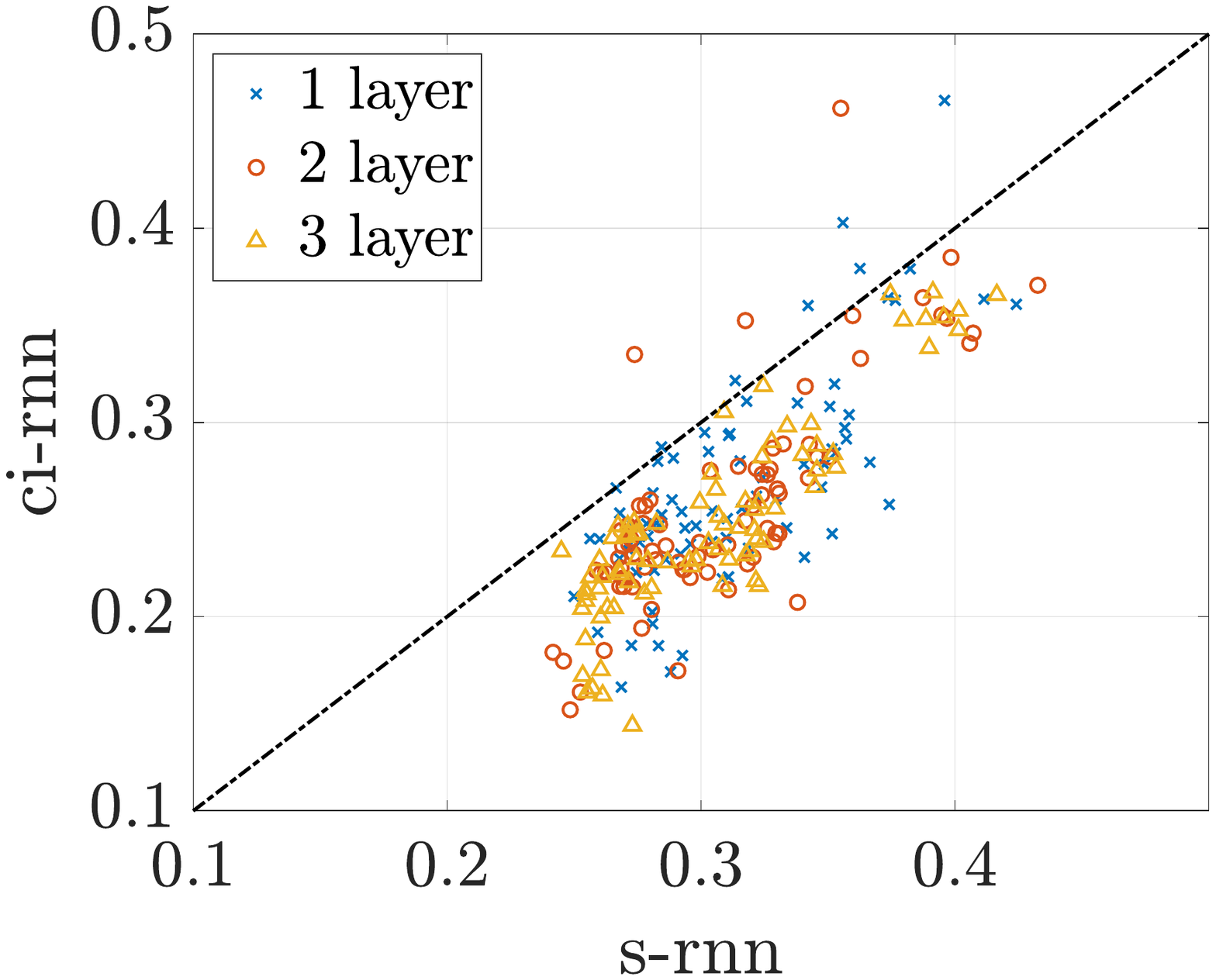}
	}
	\caption{Performance of models on gait prediction task. The subscripts refer to the number of layers in the model.}
\end{figure}

In Figure \ref{fig:training speed}, for each model, we have plotted the mean square error of the outputs of one trial on the training data versus the epochs. We observe a significant increase in the training speed of the ci-RNN compared to the s-RNN and RNN. 
We believe that this is due to the proposed initialization scheme and the increased flexibility provided by the redundant parametrization.
 In order to compare the performance of the resulting models across different participants and model outputs, we use Normalized Simulation Error (NSE) as a performance metric, calculated as:
\begin{equation}
\text{NSE} = \frac{\sum_{t}|y_t - \tilde{y}_t|^2}{\sum_{t}|\tilde{y}_t|^2}.
\end{equation}
The box-plots in Figure \ref{fig:train boxplot} and \ref{fig:test boxplot} show the average NSE across the 6 outputs for the training and test datasets for a single participant across the 9 models trained. We see that in each case the ci-RNN outperforms the models RNN and s-RNN. Additionally we also observe a number of unstable models in the RNN model set that have unbounded NSE.
Figure \ref{fig:test scatter} compares the NSE of the ci-RNN and s-RNN for all models trained and all participants. As the vast majority of the point lie beneath the line $y=x$, we can see that almost all ci-RNNs trained outperform the corresponding s-RNNs, does so in every case for the 3-layer networks.

\bibliographystyle{ieeetran} 
\bibliography{Refs_edited}

\end{document}